\documentclass{article}

\usepackage{PRIMEarxiv}

\usepackage[utf8]{inputenc} 
\usepackage[T1]{fontenc}    
\usepackage{hyperref}       
\usepackage{url}            
\usepackage{booktabs}       
\usepackage{amsfonts}       
\usepackage{nicefrac}       
\usepackage{microtype}      
\usepackage{lipsum}
\usepackage{fancyhdr}       
\usepackage{graphicx}       
\graphicspath{{media/}}     
\usepackage{latexsym}
\usepackage{amssymb}
\usepackage{amsmath}
\usepackage{amsthm}
\usepackage{algpseudocode}
\usepackage{algorithm}
\usepackage{booktabs}
\usepackage{bbm}
\usepackage{enumitem}
\usepackage{graphicx}
\usepackage{color}
\usepackage{pifont}
\newcommand{\cmark}{\ding{51}}%
\newcommand{\xmark}{\ding{55}}%

\newtheorem{theorem}{Theorem}

\newtheorem{corollary}[theorem]{Corollary}
\newtheorem{proposition}[theorem]{Proposition}

\newtheorem{definition}{Definition}
\newtheorem{remark}{Remark}

\title{QUCE: The Minimisation and Quantification of Path-Based Uncertainty for Generative Counterfactual Explanations
}

\author{
  Jamie Duell, Monika Seisenberger \\
  School of Mathematics and Computer Science \\
  Swansea University \\
  Swansea\\
  \texttt{\{jamie.duell, m.seisenberger\}@swansea.ac.uk} \\
   \And
  Hsuan Fu \\
  Department of Finance, Insurance and Real Estate \\
  Université Laval \\
  Québec City\\
  \texttt{hsuan.fu@fsa.ulaval.ca} \\
  \And 
  Xiuyi Fan\\
  Lee Kong Chian School of Medicine \\
  School of Computer Science and Engineering\\
  Nanyang Technological University \\
  Singapore\\
  \texttt{xyfan@ntu.edu.sg} \\
  }

\begin{document}
\maketitle

\begin{abstract}
Deep Neural Networks (DNNs) stand out as one of the most prominent approaches within the Machine Learning (ML) domain. The efficacy of DNNs has surged alongside recent increases in computational capacity, allowing these approaches to scale to significant complexities for addressing predictive challenges in big data. However, as the complexity of DNN models rises, interpretability diminishes. In response to this challenge, explainable models such as Adversarial Gradient Integration (AGI) leverage path-based gradients provided by DNNs to elucidate their decisions. Yet the performance of path-based explainers can be compromised when gradients exhibit irregularities during out-of-distribution path traversal. In this context, we introduce Quantified Uncertainty Counterfactual Explanations (QUCE), a method designed to mitigate out-of-distribution traversal by minimizing path uncertainty. QUCE not only quantifies uncertainty when presenting explanations but also generates more certain counterfactual examples. We showcase the performance of the QUCE method by comparing it with competing methods for both path-based explanations and generative counterfactual examples.
\end{abstract}

\section{Introduction}\label{sec1}

 Given the prevelance of big data and increased computability, the application of Deep Neural Network (DNN) methods are a commonality. However, the intricacies and depth of DNN architectures lead to results that lack inherent interpretability. In pivotal domains such as healthcare and finance, interpretability is crucial and thus the application of eXplainable Artificial Intelligence (XAI) to extract valuable insights from the DNN models is widespread \cite{10101766,xaifinreview}.

The Path-Integrated Gradients (Path-IG) \cite{IG} formulation presents axiomatic properties that are upheld solely by path-based explanation methods. The Out-of-Distribution (OoD) problem is prevalent in the application of path-based explanation methods \cite{10356434}; here the intuition is that traveling along a straight line path can incur irregular gradients and thus provide noisy attribution values \cite{GIG}. Another known limitation of many Integrated Gradient (IG) \cite{IG} based approaches is the selection of a baseline reference; thus the Adversarial Gradient Integration (AGI) \cite{AGI} method relaxes this constraint by generating baselines in adversarial classes. We note that AGI utilizes the path-based approach for generating counterfactual examples, and for this reason will be a primary baseline for our proposed method throughout this paper.

Counterfactual explanations \cite{Guidotti2022} are often presented in the form of counterfactual examples \cite{wachter,DiCE}; here the goal is to provide a counterfactual example belonging to an alternative class with respect to a reference example. Counterfactual approaches aim to answer the question: 
\begin{quote}
``Given an instance, what changes can be made to change the outcome for that instance?" 
\end{quote}
Naturally, this allows for empirical observation as to which changes could provide an alternative outcome. The argument for using counterfactual methods is often developed from a causal lens \cite{Hfler2005,Prosperi2020}. It follows that to better evaluate this causal relationship, a promising avenue is to unify feature attribution with counterfactual examples, as demonstrated by the Diverse Counterfactual Explanations (DiCE) \cite{DiCE} method. Naturally, given quantitative approaches to feature attribution calculation such as these, ideally feature attribution methods should adhere to desirable axioms across XAI literature \cite{IG,ijcai2022p90}. Thus, we aim to utilize state-of-the-art feature attribution assignment as to satisfy key axioms in our model development. Another concern with production of counterfactual examples is the production of realistic paths to successfully create a counterfactual example; therefore we shall be exploring uncertainty. 

Uncertainty quantification is not often considered when producing explanations, although some approaches have explored this. Examples include \cite{NEURIPS2021_slack} where post-hoc model-agnostic approaches such as Local Interpretable Model-Agnostic Explanations (LIME) \cite{Ribeiro2} and kernel SHapley Additive exPlanations (SHAP) \cite{Lundberg} are adapted into a Bayesian framework to model the uncertainty of the explanations produced. Since path-based methods are implementation invariant with respect to the model, the explanations will be consistent and thus there will be no variance in the explanations produced. In this way uncertainty quantification in the form of repeated runs of the XAI algorithm as elucidated in \cite{madaan2023uncertainty}, while applicable to post-hoc approximation XAI methods, will not suffice for implementation-invariant models. Autoencoder-based frameworks have also been used to measure uncertainty for both machine learning predictions and explanations \cite{antoran2021getting}, with integration of uncertainty seen in the production of counterfactual examples \cite{CLEAR,AECF}. Inherently, autoencoder approaches provide a more suitable basis for attribution settings by evaluating uncertainty in explanations with respect to the uncertainty inherent in the data. The standard autoencoder approach evaluates the reconstruction error, which is often utilized in work surrounding anomaly detection \cite{Torabi2023,angiulli2023reconstruction}; instead, we explore the use of a variational autoencoder (VAE) for variational inference, and thus investigate counterfactuals generated with respect to our approximation of the true data distribution. 

To address the above constraints, we propose the Quantified Uncertainty (Path-Based) Counterfactual Explanations (QUCE) method. The focus of the proposed method is three-fold. We aim to
\begin{itemize}
    \item minimize uncertainty and thus maximize the extent to which the generated paths and counterfactual examples are within distribution;
    \item relax the straight-line path constraints of Integrated Gradients;
    \item provide uncertainty quantification for counterfactual paths and counterfactual feature attribution.
\end{itemize}
In this work, we focus on the minimization of uncertain paths for counterfactual generation with quantifiable uncertainty measures on the generated counterfactual. QUCE's learning process relaxes IG's straight-line path restrictions as part of the generative process. 

Intuitively, it is unclear in many scenarios if one single best path toward an alternative outcome exists; for example a patient's treatment path may be unclear \cite{Beutler2016}, or there may be many viable paths to achieve the same outcome \cite{Bull2020}. Therefore, QUCE utilises both a single and multiple-paths approach, so we can observe a generalized explanation over all paths for an instance in obtaining a desired class and likewise inspect many example paths. From the multiple-paths approach, we are able to gauge a general approximation over many piecewise linear paths of the most important features in obtaining a desired counterfactual outcome,  each path independently aiming to minimize the uncertainty in its generative process and thus provide greater in-distribution interpolation. Similarly, we present the optimisation over the key metrics --  \emph{proximity} \cite{Guidotti2022}, \emph{validity} \cite{Guidotti2022} and \emph{uncertainty} \cite{9707594}. We provide an overview of some generative counterfactual methods and evaluate which account for the aforementioned metrics in table \ref{table:path_uncertainty_eval}. 

\begin{table}[t]
\centering
\caption{An overview of generative counterfactual methods and their consideration of key metrics. Here we observe of the three metrics QUCE is the method that accounts for all three.}
\begin{tabular*}{\columnwidth}{c @{\extracolsep{\fill}} ccc}
\toprule
\textbf{Properties} & Proximity & Validity & Uncertainty  \\
\midrule
QUCE & \cmark & \cmark & \cmark \\ 
DiCE & \cmark  & \cmark  & \xmark  \\ 
AGI & \cmark & \cmark & \xmark \\
\end{tabular*}
\label{table:path_uncertainty_eval}
\end{table}



\section{Explainable AI and Counterfactuals}
Counterfactual explanations can be presented both in the form of counterfactual examples \cite{DiCE,wachter} and also via counterfactual feature attribution \cite{DiCE}.  
\begin{definition}[Counterfactual Example]\label{counterfactual_example}
Given a probabilistic classifier $f : \mathbbm{R}^J \rightarrow \{0,1\}$, differentiable probabilistic function $F : \mathbbm{R}^J \rightarrow [0,1]$, and class $\tau = 0 \mbox{ or } 1$ an instance $\mathbf{x} = \langle x^1, \ldots, x^J \rangle \in X$ where $X \in \mathbbm{R}^{N \times J}$ and a classification threshold $\vartheta \in [0,1]$ such that 
\begin{align}
    f(\mathbf{x}) = \begin{cases}
\neg \tau, \mbox{if } F(\mathbf{x}) \leq \vartheta \\
\tau, \mbox{otherwise.}
\end{cases}
\end{align}
Then a counterfactual example of $\mathbf{x}$ is some $\mathbf{x}^c$ where $f(\mathbf{x}^c) \neq f(\mathbf{x})$.
\end{definition}
Counterfactual examples are often produced through the use of a learned generative function, examples of this can be seen in the work of \cite{CLEAR}. Extending this, we define a \emph{counterfactual generator}. 
\begin{definition}[Counterfactual Generator]
Given an instance $\mathbf{x} \in X$ and a classifier $f$, a counterfactual generator is a function $\mathcal{G} : \mathbbm{R}^J \rightarrow \mathbbm{R}^J$ that takes an instance $\mathbf{x}$ and returns a counterfactual example $\mathbf{x}^c \notin X$.
\end{definition}
Feature attribution is a common form of XAI, and the feature attribution approach is seen in many methods \cite{Lundberg,raghav,negflex,dejl2023cafe}. We continue by defining \emph{feature attribution}:
\begin{definition}[Feature Attribution]
Given an instance $\mathbf{x}$, a feature attribution method is a function $\Phi : \mathbbm{R}^J \rightarrow \mathbbm{R}^J$, where $\Phi$ takes an instance $\mathbf{x}$ and returns a vector of feature attribution values $\langle \phi^1, \ldots, \phi^J \rangle$.
\end{definition}
Concatenating prior feature attribution methods and counterfactual examples leads to \emph{counterfactual feature attribution} (CFA) which we define as follows: 
\begin{definition}[Counterfactual Feature Attribution]
 A counterfactual feature attribution method is a function $\Phi_{CF} : \mathbbm{R}^J \times \mathbbm{R}^J \rightarrow \mathbbm{R}^J$ that takes an instance $\mathbf{x}$ and a counterfactual example $\mathbf{x}^c$ and returns a vector of counterfactual feature attribution values $\langle \phi_{CF}^1, \ldots, \phi_{CF}^J \rangle$.    
\end{definition}

\section{Axioms for Path-Based Explainers}\label{sec:axioms}

 Here we informally introduce axioms used in XAI literature. In the seminal work of \cite{IG}, the authors introduce a set of axioms which play a foundational role for the development of path-based feature attribution methods. Informally, these encompass:
\begin{itemize}
    \item \textbf{Completeness}: The difference in prediction between the baseline and input should be equal to the sum of feature attribution values. 
    \item \textbf{Sensitivity(a)}: For every input and baseline that differ in one feature and for which the subsequent prediction is different, feature attribution should only be given to that one feature. 
    \item \textbf{Sensitivity(b)}: If the neural network is not mathematically dependent on one feature, the feature attribution assigned to that feature should be 0.
    \item \textbf{Implementation Invariance}: Two functionally equivalent neural networks should produce the same feature attribution as an explanation. 
\end{itemize}
The above axiomatic guarantees are for path-based explainers and for this reason we utilise the path-based formulation in our work. 

\section{Proposed Model: QUCE}
\subsection{Generating Counterfactuals with QUCE}
To generate counterfactuals we propose a three-part objective function with a composite  weighting vector $\Lambda = \langle \lambda_1, \lambda_2, \lambda_3 \rangle$, where each $\lambda \in \Lambda$ is an independent tolerance (weight) used to determine the influence of each part of the joint objective function presented in equation \ref{loss_function}. By minimizing the objective function we obtain the generated counterfactual $\mathbf{x}^c$. Informally, our objective function is composed of three parts:
\begin{itemize}
    \item $\mathcal{L}_{pr}$, for the maximization the probability towards the desired class;
    \item $\mathcal{L}_{\delta}$, to minimize the distance between the  instance and a generated counterfactual;
    \item $\mathcal{L}_{\epsilon}$, to minimize the uncertainty of both the generated paths and generated counterfactual examples.
\end{itemize}
Combining these terms with our weighting vector, we have
\begin{align}\label{loss_function}
     \mathcal{G}(\mathbf{x}) &= \underset{\mathbf{x}^c}{\arg \min} \  \lambda_1\mathcal{L}_{pr} + \lambda_2\mathcal{L}_{\delta} + \lambda_3\mathcal{L}_{\epsilon}.
\end{align}
Having constructed our generative objective function, we provide further notation to  illustrate the learning process. First, we consider an iterative learning process such as gradient descent on $\mathbf{x}$ to produce a path from $\mathbf{x}$ to a generated $\mathbf{x}^c$. Thus we aim to minimize the developed function $\mathcal{G}(\mathbf{x})$ through the gradient descent approach (and variants e.g. ADAM). We initially let $\mathbf{x}^c = \mathbf{x}$; $\mathbf{x}^c$ is updated via
\begin{align*}
     &\mathbf{x}^c \Leftarrow \mathbf{x}_{\Delta_i},\\
     &\Delta_i = \varphi \nabla_{\mathbf{x}^c}(\mathcal{G}(\mathbf{x})), \\ 
    &\mathbf{x}_{\Delta_i} = \mathbf{x} - \Delta_i.
\end{align*}
Let $\mathbf{x}$ be updated on a loop over $i (0 \leq i \leq n)$ iterations; when $i = n$ we let have our $\mathbf{x}^c$ indicating our generated point. Here $\varphi$ represents the ``learning rate," a small positive multiplier value $\varphi \in [0,1] : \varphi << 1$. We store each update on $\mathbf{x}^c$ as $\mathbf{x}^{\Delta}$ = $\langle \mathbf{x}_{\Delta_0}, \ldots, \mathbf{x}_{\Delta_n} \rangle$.

\subsection{Finding Counterfactuals}

\subsubsection{Valid Counterfactuals}
A key concept in finding counterfactual examples is ensuring that the counterfactual is indeed \emph{valid}, and thus we aim to produce counterfactual examples that belong to a counterfactual class.
\begin{remark}\label{remark2}
    For the sake of simplicity we let $f(\mathbf{x}) = \neg \tau$ throughout the remainder of this work; thus $\tau$ is the counterfactual class for which we aim to generate some $\mathbf{x}^c$ such that $f(\mathbf{x}^c) = \tau$. 
\end{remark}
In light of remark \ref{remark2}, given a target counterfactual class $\tau$ and probabilistic decision threshold $\vartheta$ we aim to find an instance $\mathbf{x}^c$ that satisfies  
\begin{align}\label{condition:probability}
F(\tau \vert \mathbf{x}^c) \geq \vartheta.
\end{align}
Here $\vartheta$ is a probability threshold for the target class $\tau$. Thus, we need a generator $\mathcal{G}$ that satisfies the condition in equation \ref{condition:probability}. To achieve this, we can maximize the likelihood of an instance belonging to a class $\tau$. The maximum log likelihood criterion is defined as:
\begin{align*}
    &\mathcal{L}_{pr} =  \bigg[  log[F( \tau \vert \mathbf{x}^c)]\bigg]
\end{align*}
such that, we only accept counterfactuals where $F(\cdot\vert\cdot) \geq \vartheta$. For optimisation we can rewrite this as a minimization problem, instead minimizing the negative log-likelihood:
\begin{align*}
    &\mathcal{L}_{pr} =  \bigg[  -log[F( \tau \vert \mathbf{x}^c)]\bigg]. 
\end{align*}
This constitutes the constrained optimisation problem with respect to some target class $\tau$. 
\subsubsection{Proximity for Counterfactuals}
Given an instance, in the production of counterfactual examples we often aim to find a counterfactual example that is ``similar" in feature space to the instance. This is often termed \emph{proximity}. 

In this work, we use the $l_2$ norm as the proposed model focuses on producing counterfactuals from continuous features, defining proximity as follows:

\begin{definition}[Proximity]
    Given an instance $\mathbf{x}$ and its counterfactual example $\mathbf{x}^c$, the proximity between the two instances is given by 
    \begin{align}
    \mathcal{L}_{\delta} = \bigg[\frac{1}{2} \vert\vert \mathbf{x}^c - \mathbf{x} \vert\vert^2 \bigg].
    \end{align}
\end{definition}

\subsubsection{Minimally Uncertain Counterfactuals}
To maximize the certainty of counterfactual examples, we examine their complement—namely, the uncertainty associated with a counterfactual example. To explore this, we establish the concept of \emph{counterfactual uncertainty}.
Informally, we consider uncertainty to be the Evidence Lower Bound (ELBO) as measured by a Variational Autoencoder (VAE) framework. The objective function ELBO is comprised of two components, namely the Kullbeck--Leibler (KL) divergence and a reconstruction loss. This is defined as the following: 
\begin{align}\label{VAEloss}
    &\operatorname{VAELoss}(\mathbf{x}) = \mathbbm{E}_{q_{\theta}}[log \ q_{\theta}(\mathbf{z} \vert \mathbf{x}) - log \ p(\mathbf{z})] - \mathbbm{E}_{q_{\theta}} \ [log \ p_{\psi}(\mathbf{x} \vert \mathbf{z})] \nonumber
\end{align}
where $p$ and $q$ are probability distributions and $\mathbf{z}$ is the latent representation of $\mathbf{x}$. The aim is to find a $\theta$ that successfully models the true training data distribution $\psi$, and thereby satisfying the following minimization problem:
\begin{align}
    \{\theta^*, \psi^* \} = \underset{\theta, \psi}{\arg \min}\operatorname{VAELoss}(\mathbf{x})
\end{align}
Intuitively, ELBO aims to produce a  $p$ and $q$ that are as close as possible while simultaneously optimizing the mapping of $\mathbf{z}$ back to its original representation $\mathbf{x}$, through learning to model the distributions with the parameters $\theta$ and $\psi$. Thus it suffices to say that if the reconstructing loss is minimized, the decoded $\mathbf{z}$ should successfully map back to $\mathbf{x}$. Posterior to the pre-trained VAE, we now have a fixed representation shaping our $p$ and $q$ distributions with our parameters $\theta^*$ and $\psi^*$ and thus we can provide counterfactual uncertainty as:
\begin{definition}[Counterfactual Uncertainty]\label{def:uncertainty}
Following equation \ref{VAEloss}, we simply take the loss without minimization to be the counterfactual uncertainty, thus given our fixed parameters $\theta^*$ and $\psi^*$, we have
\begin{align}
\mathcal{L}_{\epsilon} = 
     \mathbbm{E}_{q_{\theta^*}}[log \ {q_{\theta^*}}(\mathbf{z} \vert \mathbf{x}^c) - log \ p(\mathbf{z})] - \mathbbm{E}_{q_{\theta^*}} \ [log \ p_{\psi^{*}}(\mathbf{x}^c \vert \mathbf{z})]. \nonumber
\end{align}
\end{definition}

\subsection{Uncertainty in Counterfactual Explanations}
Definition \ref{def:uncertainty} allows for the evaluation of new generated instances with a measure of how ``good" the fit of the new instance is with respect to the training data distribution, and similarly how well a path fits into the data distribution.

From definition \ref{def:uncertainty} we have a quantifiable measure of uncertainty for the generated counterfactual $\mathbf{x}^c$. Expressing this in vector form as a difference, which is needed for modifying and updating our generated $\mathbf{x}^c$ to adjust for underlying uncertainty, we define \emph{Feature-wise Counterfactual Uncertainty} by simply looking at the reconstruction error, defined for simplicity as follows:
\begin{definition}[Feature-wise Counterfactual Uncertainty]
Given counterfactual uncertainty, we can rewrite a reconstruction error equivalent at a feature level by calculating a vector of differences $\epsilon_\mathbf{d}$ such that 
    \begin{align}
    \epsilon_\mathbf{d} = \langle \vert d^1 \vert, \ldots, \vert d^J \vert \rangle \ : \ \langle d^1, \ldots, d^J \rangle =  \mathbf{d}; \\ \mbox{where }\ \mathbf{d}  = \mathbf{x}^c - \hat{\mathbf{x}}^c.
    \end{align}
\end{definition}
With this representation, we can then successfully update $\mathbf{x}^c$ by both adding and subtracting this vector of feature-wise counterfactual uncertainty as given by the reconstruction error, and thus we can calculate \emph{Counterfactual Explanation Uncertainty}.  
\begin{definition}[Counterfactual Explanation Uncertainty]
Given a CFA $\Phi_{CF}$ and the feature-wise counterfactual uncertainty $\epsilon_\mathbf{d}$, the counterfactual explanation uncertainty is given by
\begin{align}
    \Phi_{CF}^{\epsilon_\mathbf{d}} =  \Phi_{CF}(\mathbf{x}^c \pm \epsilon_\mathbf{d}, \mathbf{x}).
\end{align}
\end{definition}

\subsection{Path Explanations}
\begin{figure*}[t]
    \centering  \includegraphics[scale=0.5]{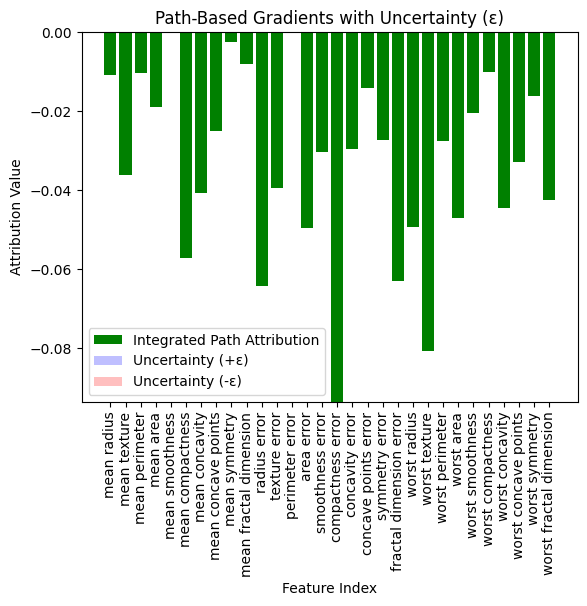}  \includegraphics[scale=0.5]{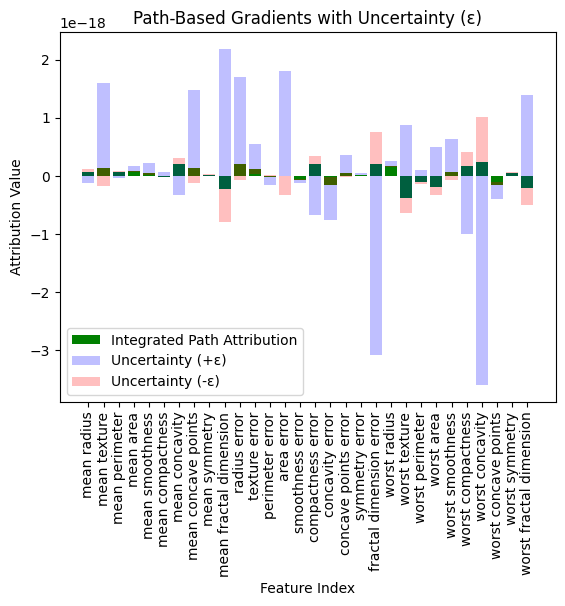}
    \caption{Here we illustrate two explanations produced on the Wisconsin Breast Cancer Dataset given by the proposed QUCE method. We observe how each feature influenced the change in the prediction as we attempted to generate a counterfactual example. This demonstrates how explanations are presented through the QUCE method, with the benefit of observing uncertainty in explanations presented. We see the left explanation has almost no uncertainty in generated explanation, whereas the right image demonstrates a large degree of uncertainty in the generated counterfactual explanation.}
    \label{fig:example_explanations}
\end{figure*}
The Path-Integrated Gradients \cite{IG,AGI,BatchIG,GIG,IDGI} formulation is the only approach to our knowledge within the landscape of feature attribution methods that satisfies all the feature attribution axioms in section \ref{sec:axioms}. Therefore we adopt the path integral formulation and relax the straight-line constraint seen in IG. To achieve this, recall the set of learned updates on $\mathbf{x}$, namely $\mathbf{x}^{\Delta}$. It follows that we can produce explanations over $\mathbf{x}^{\Delta}$ with respect to a classifier $f$. 

Formally, let the function $F$ be a continuously differentiable function, the QUCE explanation takes the path integral formulation such that given a smooth function $\psi = \langle \psi^1, \ldots, \psi^J \rangle : [0,1] \rightarrow \mathbb{R}^J$ defining a path in $\mathbb{R}^J$, where $\psi(\alpha)$ is a point along a path at $\alpha \in [0,1]$ with $\psi(0) = \mathbf{x}_{\Delta_0}$ and $\psi(1) = \mathbf{x}_{\Delta_{n}}$, the single-path QUCE explainer is defined as
\begin{align}
    \Phi_{\mbox{QUCE}}(\mathbf{x}^{\Delta}) := 
    \int^{\mathbf{x}_{\Delta_n}}_{\mathbf{x}_{\Delta_0}} \nabla F(\psi(\alpha)) \cdot \psi^\prime(\alpha) d \alpha.    
\end{align}
It follows that explanation uncertainty with respect to a single generated  counterfactual $\mathbf{x}^c$ is given as 
\begin{align}\label{eqn:QUCE_uncertainty}
    \Phi^{\pm\epsilon_\mathbf{d}}_{\mbox{QUCE}}(\mathbf{x}^c) :=  
    ((\mathbf{x}^c \pm \epsilon_\mathbf{d})  - \mathbf{x}^c)\times \bigg( \int^{\mathbf{x}^c\pm\epsilon_\mathbf{d}}_{\mathbf{x}^c} \nabla F(\psi(\alpha)) d \alpha\bigg).
\end{align}
We now show through the proof of proposition \ref{Riemann_prop_single} that the QUCE explanations are easily computed.

\begin{proposition}\label{Riemann_prop_single}
    The QUCE explainer has a computable Riemann approximation solution for each feature.
\end{proposition}
\begin{proof}
    Proof provided in the supplementary material.
\end{proof}
Both attributed values from equation \ref{eqn:QUCE_uncertainty} illustrate uncertainty in feature attribution values given by the QUCE explainer. In proposition \ref{prop:uncertainty} and its associated proof we show the implications of weighting uncertainty. 
\begin{definition}[$\lambda$-tolerance]
    An increase in $\lambda$-tolerance refers to the reduction of any component $\lambda$ of $\Lambda$. Likewise, a decrease in $\lambda$-tolerance refers to the increase of any component $\lambda$ of $\Lambda$.
\end{definition}
Naturally, as the weight for uncertainty decreases, we lean towards an increased tolerance of the effects of uncertainty in our explanations. The reasoning behind this is that we may sometimes accept a higher degree of uncertainty, depending on the purpose of generating counterfactuals and the volatility of the task.
\begin{proposition}\label{prop:uncertainty}
    Increasing the $\lambda$-tolerance of uncertainty
    provides a more flexible search space for possible paths to a generative counterfactual example. 
\end{proposition}
\begin{proof}
This is trivial subject to $\lambda_{3}$ approaching zero. Proof provided in the supplementary material.
\end{proof}
Due to the stochastic nature of our model under nondeterministic variants of gradient descent (e.g. optimization with stochastic gradient descent \cite{4308316}), and with potentially multiple minima (e.g. we may have two points equally ``close" to the decision bound with different values),  we consider a set of generated counterfactual examples to be given as $C = \langle \mathbf{x}^c_{1, 1}, \ldots, \mathbf{x}^c_{1, k} \rangle$, where $k$ is the number of generated counterfactual examples over some  set $\mathbf{x}$. Given $C$, we can accumulate attribution over many counterfactuals by avoiding the specification of $\mathbf{x}^c$; our lower limit is implicitly assumed to be our instance to explain $\mathbf{x}$, so that we have 
\begin{align}
        \Phi_{\mbox{exQUCE}}(\mathbf{x}) :=  \int_{\mathbf{x}^c} \bigg(\Phi_{\mbox{QUCE}}(\mathbf{x}^{\Delta}) \bigg) p_C(\mathbf{x}^c) d\mathbf{x}^c
\end{align}
where we integrate over $p_C$ the  distribution of $C$ for all $\mathbf{x}^c \in C$ and we can instead rewrite the integral as an expectation as follows:
\begin{align}
    \Phi_{\mbox{exQUCE}}(\mathbf{x}) :=  \underset{{\mathbf{x}^c \sim C}}{\mathbbm{E}} \bigg[\Phi_{\mbox{QUCE}}(\mathbf{x}^{\Delta})  \bigg].
\end{align}
Here we let $\alpha \sim \mathcal{U}(0,1)$ indicate interpolation over $\alpha$ for $m$ counterfactual steps in the generator function. Informally, we get the expectation of the gradients over the piecewise linear path between counterfactual steps of the generator. We take a similar approach to the Expected Gradients \cite{Erion2021} formulation, except that instead of sampling from a background set of baselines, we sample from a set of generative counterfactual examples.
We make two arguments as to why we use this approach: 
\begin{itemize}
    \item In explaining a counterfactual outcome, we do not know the specific path taken and thus we can average over many paths. 
    \item We can invert the path to explain $\mathbf{x}$ and therefore we can have many generative baselines. This relaxes the specified baseline of many existing path-based explanation methods.
\end{itemize}
To further extend on the axiomatic guarantees of path-based explainers, we show via proposition \ref{proposition_expectedatt} that completeness holds when working with the many-paths approach for expected values. 
\begin{proposition}\label{proposition_expectedatt}
Given the function $\Phi_{\mbox{exQUCE}}(\mathbf{x})$, the expected difference in prediction probabilities between generated counterfactuals in the set $C$ with respect to the prediction probability given by $F(\mathbf{x})$, the following equality holds:
\begin{align}\label{p1}
&\underset{{\mathbf{x}^c \sim C}}{\mathbbm{E}} \bigg[\Phi_{\mbox{QUCE}}(\mathbf{x}^{\Delta})  \bigg] \\ \label{p2}&= \underset{{\mathbf{x}^c \sim C}}{\mathbbm{E}} \bigg[F(\mathbf{x}^c) - F(\mathbf{x}) \bigg]    
\end{align}
\end{proposition}
\begin{proof}
This is a direct consequence of the \emph{completeness} axiom; the proof is provided in the supplementary material.
\end{proof}
It follows that the expected value approach is equally computable and is a direct extension of proposition \ref{prop:uncertainty}. We illustrate this in corollary \ref{expectedquce_comp}'s simple proof.  
\begin{corollary}\label{expectedquce_comp}
    The expected QUCE variant has a computable Riemann approximation solution for each feature.
\end{corollary}
\begin{proof}
    This follows from proposition \ref{Riemann_prop_single}; the proof is provided in the supplementary material.
\end{proof}
Going further, we demonstrate monotonic relationships for generated counterfactual instances that are given by the multiple paths approach. This is a further consequence of the completeness axiom and is expressed in corollary \ref{expected_mono}.  
\begin{corollary}\label{expected_mono} Given two sets of counterfactual examples $C^1$ and $C^2$ for an instance $\mathbf{x}$,
\begin{align*}
&\mbox{if} \ \underset{{\mathbf{x}^c \sim C^1}}{\mathbbm{E}}\bigg[ F(\mathbf{x}^c) - F(\mathbf{x}) \bigg] \leq \underset{{\mathbf{x}^c \sim C^2}}{\mathbbm{E}}\bigg[ F(\mathbf{x}^c) - F(\mathbf{x}) \bigg], \\ &\mbox{then} \
\underset{{\mathbf{x}^c \sim C^1}}{\mathbbm{E}}\bigg[ \Phi_{\mbox{exQUCE}}(\mathbf{x}) \bigg] \leq \underset{{\mathbf{x}^c \sim C^2}}{\mathbbm{E}}\bigg[ \Phi_{\mbox{exQUCE}}(\mathbf{x}) \bigg].
\end{align*}
\end{corollary}
\begin{proof}
    This is a direct consequence of proposition \ref{proposition_expectedatt} and the \emph{completeness} axiom. 
\end{proof}
Given we can compute many-paths explanations, it follows that we can also take the expected gradients for the explanation uncertainty computed by QUCE along each path, such that 
\begin{align}
\Phi^{\pm \epsilon_\mathbf{d}}_{\mbox{exQUCE}}(\mathbf{x}^c) := \underset{\mathbf{x}^c \sim C}{\mathbbm{E}} \bigg[ \Phi^{\pm\epsilon_\mathbf{d}}_{\mbox{QUCE}}(\mathbf{x}^c) \bigg].
\end{align}

\section{Experimental Setup}
\begin{table*}[t]
\centering
\caption{Comparison of the average path uncertainty on the generated counterfactual instances. This is experimented over 100 instances from the training and testing sets of each dataset. Here we have 1000 steps (path interpolation instances) for the Riemann approximation of every path-based approach, thus effectively 100$\times$1000 instances. Here the lower value the better. The proposed QUCE method shows superior performance on average when comparing counterfactual path-based approaches.}
\begin{tabular*}{\textwidth}{c @{\extracolsep{\fill}} ccccccccc}
\toprule
\textbf{Path $\mathcal{L}_{\epsilon}$} & Lung & Breast & Skin & Lymph & Rectal & COVID & W-BC \\
\midrule
\textbf{Train} \\
\midrule
QUCE & \textbf{0.92$\pm$0.32} & \textbf{0.82$\pm$0.26} & \textbf{0.94$\pm$0.41} & \textbf{0.74$\pm$0.06} & \textbf{0.86$\pm$0.28} & \textbf{1.36$\pm$0.15} & \textbf{0.82$\pm$0.16}  \\ 
IG-QUCE & 0.95$\pm$0.32  & 0.84$\pm$0.27  & 0.96$\pm$0.41 & 0.76$\pm$0.06 & 0.86$\pm$0.29 & 1.39$\pm$0.11 & 0.84$\pm$0.20 \\ 
AGI & 1.94$\pm$1.86 & 1.49$\pm$0.95 & 1.80$\pm$1.47  & 0.92$\pm$0.28 & 2.19$\pm$2.23 & 2.01$\pm$0.15 & 0.93$\pm$0.36 \\ 
\midrule
\textbf{Test} \\
\midrule 
QUCE & \textbf{0.82$\pm$0.34} & \textbf{0.91$\pm$0.28}  & \textbf{0.82$\pm$0.31} & \textbf{0.69$\pm$0.19} & 0.80$\pm$0.29 & \textbf{0.61$\pm$0.07} & \textbf{0.83$\pm$0.29}  \\ 
IG-QUCE & 0.83$\pm$0.33 & \textbf{0.91$\pm$0.28} & \textbf{0.82$\pm$0.31} & 0.70$\pm$0.19& \textbf{0.79$\pm$0.29} & 0.67$\pm$0.05 & 0.85$\pm$0.33 \\ 
AGI & 1.22$\pm$1.16 &1.83$\pm$0.97 & 1.34$\pm$1.20 & 0.89$\pm$0.28 & 1.57$\pm$1.35 
 & 0.82$\pm$0.11 & 0.96$\pm$0.55 \\  
\end{tabular*}
\label{table:path_uncertainty_eval}
\end{table*}
\begin{table*}[t]
\centering
\caption{Comparison of the average VAE loss for generated counterfactual examples. This is experimented over 100 instances on each dataset. Here we observe that the proposed QUCE method performs best across all datasets.}
\begin{tabular*}{\textwidth}{c @{\extracolsep{\fill}} ccccccc}
\toprule
\textbf{Counterfactual $\mathcal{L}_{\epsilon}$} & Lung & Breast & Skin & Lymph & Rectal & COVID & W-BC \\
\midrule
\textbf{Train} \\ 
\midrule 
QUCE & \textbf{1.01} & \textbf{0.78}  & \textbf{0.97} & \textbf{0.71}  & \textbf{0.85} & \textbf{1.25} & \textbf{0.93} \\ 
DiCE &  1.97 & 1.02 & 1.48  & 1.10 & 1.27 & 1.57 & 3.63  \\ 
AGI & 2.93 & 2.07  & 2.61 & 1.01 & 3.48 & 2.40 & 1.22 \\
\midrule 
\textbf{Test} \\
\midrule 
QUCE & \textbf{0.93}  & \textbf{0.91}  & \textbf{0.86} & \textbf{0.67} & \textbf{0.83} & \textbf{0.76}  & \textbf{1.08}\\ 
DiCE & 1.95  &  1.09 & 1.33 & 1.18 & 1.22 & 0.92  & 3.00 \\ 
AGI & 1.68 & 2.75 & 1.87 & 1.02 & 2.38 & 0.84 & 1.41  \\
\end{tabular*}
\label{table:cf_uncertainty_tbl}
\end{table*}
\begin{table*}[t]
\centering
\caption{Comparison of the average sum of feature-wise reconstruction error between original instances and their generated counterfactual examples. This is experimented on 100 instances for each dataset. Here we observe that the QUCE method performs best in generating counterfactuals with minimal uncertainty across all datasets.}
\begin{tabular*}{\textwidth}{c @{\extracolsep{\fill}} ccccccc}
\toprule
\textbf{Counterfactual Reconstruction} & Lung & Breast & Skin & Lymph & Rectal & COVID & W-BC \\
\midrule
\textbf{Train} \\ 
\midrule 
QUCE & \textbf{0.95} & \textbf{0.73}  & \textbf{0.90} & \textbf{0.66}  & \textbf{0.78} & \textbf{1.16} & \textbf{0.73} \\ 
DiCE & 1.80  & 0.95  & 1.33  & 1.01 & 1.18 & 1.38 & 1.04 \\ 
AGI & 2.55  & 1.89  & 2.25 & 0.91 & 2.91 & 2.19 & 0.80 \\
\midrule 
\textbf{Test} \\
\midrule 
QUCE & \textbf{0.88}  & \textbf{0.85}  & \textbf{0.80} & \textbf{0.63} & \textbf{0.78} & \textbf{0.57}  & \textbf{0.76}\\ 
DiCE & 1.81 & 0.95  & 1.32 & 0.99 & 1.18 & 1.38 & 1.04 \\ 
AGI & 1.53 & 2.41 & 1.66 & 0.92 & 2.08 &  0.79 & 0.81 \\
\end{tabular*}
\label{table:cf_reconstruction_tbl}
\end{table*}

\subsection{Datasets}
\subsubsection{The Simulacrum} The Simulacrum\footnote{https://simulacrum.healthdatainsight.org.uk/} is a synthetic dataset used in this study, the Simulacrum is a large dataset developed by Health Data Insight CiC and derived from anonymous cancer data provided by the National Disease Registration Service, NHS England. We produce five subsets of patient records based on ICD-10 codes corresponding to lung cancer, breast cancer, skin cancer, lymphoma and rectal cancer. These datasets are organised as survival time classification problems, where patients are predicted a survival time of either at least 6 months or less than 6 months.  

\subsubsection{COVID Rate of Infection}
The COVID rate of infection dataset contains details on control measures, temperature, humidity and the daily rate of infection for different regions of the UK. Details on data collection are provided in \cite{exmed}. This dataset is a binary classification task identifying an increased rate of infection against a non-increased rate of infection. 

\subsubsection{Wisconsin Breast Cancer}
The Wisconsin Breast Cancer (W-BC) \cite{misc_breast_cancer_wisconsin_(diagnostic)_17} dataset, provided in the scikit-learn library\footnote{\url{https://scikit-learn.org/stable/modules/generated/sklearn.datasets.load_breast_cancer.html}}, is a binary classification dataset that classifies malginant and benign tumours given a set of independent features from breast mass measurements. 

\subsection{Baseline Methods}
For comparison, we consider a selection of methods that aim to generate counterfactual examples and also a collection of path-based explainers. 

\subsubsection{Diverse Counterfactual Explanations}
DiCE, a counterfactual generator, provides feature attribution values for an instance with respect to its counterfactual examples. We use the DiCE method as a comparison for generating counterfactual examples, as DiCE is not a path-based explainer, we can only compare the generated counterfactuals. 

\subsubsection{Integrated Gradients}
The IG feature attribution method produces explanations for instances in a given dataset. To achieve this, the approach is to integrate over the gradients of a straight-line path from an all-zero baseline vector to the instance to be explained. We modify this in our experiments so that our baseline becomes the instance to be explained, while the target instance is the counterfactual generated by QUCE, so we can evaluate the straight-line path solution against the QUCE-generated path.  

\subsubsection{Adversarial Gradient Integration}
The Adversarial Gradient Integration (AGI) \cite{AGI} approach provides an alternative form of feature attribution that also produces generative counterfactual examples. The AGI method is the only path-based generative counterfactual method currently available to our knowledge and thus forms the primary focus of our comparison. We use the Individual AGI algorithm presented in their paper. 

\section{Quantitative Evaluation}
To evaluate the implementation of generative counterfactual examples, we propose using the VAE loss to determine how well the counterfactual examples fit to the underlying data distribution. 

\subsection{Path-Based Uncertainty Comparison}
To evaluate the QUCE method, we provide a comparison of uncertainty along a path. To do this, we use a pre-trained VAE, feeding all generated instances along any path into the VAE to determine the reconstruction error for all given instances along a path. The intuition behind this is that a smaller reconstruction error is associated with a path that better follows the data distribution and is therefore more ``realistic". 

In Table \ref{table:path_uncertainty_eval} we evaluate the path uncertainty across 100 instances from both the training and test data. From this, we observe that the QUCE method provides paths that better follow the data distribution when compared against both IG and AGI on average. Here we reiterate that IG is used with the generated QUCE instance, which already aims to minimize VAE loss, and is therefore used to show the minor differences in relaxing the straight-line path requirements, although this may not always be necessary. 

\subsection{Counterfactual Uncertainty}
We use counterfactual uncertainty as a measure to evaluate generative counterfactual examples given by the QUCE method. This is a simple measure of the average reconstruction error across the generative counterfactual examples across each instance in a dataset. We measure this against the DiCE and AGI methods, as both provide counterfactual examples in their generative process. 

In Table \ref{table:cf_uncertainty_tbl} we present the counterfactual uncertainty over 100 instances from both the training and test datasets over each of the datasets. Here, we observe that QUCE provides a lower value with respect to uncertainty measurements compared to both the DiCE and AGI methods, implying that the instances generated by QUCE better follow the data distribution and can thus be thought of as more likely, subject to the dataset. Further experiments on the deletion game seen in \cite{Yang_Akhtar_Wen_Mian_2023,pmlr-v202-akhtar23a}, reconstruction error and a theoretical evaluation against further explainability axioms presented in \cite{ijcai2022p90} are provided in the supplementary material. 

\subsection{Counterfactual Reconstruction Error}
In addition to evaluating the VAE loss, we also analyze the average reconstruction error per instance across all 100 instances on both the training and test datasets. This highlights the closeness of the reconstructed sample against the ground truth counterfactual generated by different methods. In Table \ref{table:cf_reconstruction_tbl} we observe that our proposed QUCE method provides better reproduced counterfactuals through the VAE than the DiCE or AGI methods. 

\section{Conclusion}
In this paper, we provide a novel approach that combines generative counterfactual methods and path-based explainers, minimizing uncertainty along generated paths and for generated counterfactual examples. We provide an analysis of the proposed QUCE method on path uncertainty, generative counterfactual example uncertainty, and proximity. Our approach provides paths that are less uncertain in their interpolations, so that more reliable gradients and explanations can be extracted. Similarly, we provide a clear explanation of uncertainty, including when and where it exists, as seen in the example explanations provided in Figure \ref{fig:example_explanations}. 

In future work we aim to relax the assumption that all feature values are continuous to provide more realistic and reliable generated examples. Exploration of optimal parameters in the QUCE framework is currently a manual process. Automating this approach would provide greater flexibility and ease of application upon distributing the method to end users.


\appendix 
\section{Computing QUCE Explanations}
\textbf{Proposition 1.}
\emph{The QUCE explainer has a computable Riemann approximation solution for each feature.} 
\begin{proof}
Given an instance $\mathbf{x}$ that is the origin of our instance for our counterfactual explanation, we consider the $\mathbf{x}^{\Delta}$ steps produced by our learning process. In order to compute QUCE for feature explanations, we deconstruct the definition to focus on a single step of QUCE. This can be directly extracted from the additive property of integrals. Recall the definition:
\begin{align*}
    \Phi_{\mbox{QUCE}}(\mathbf{x}^{\Delta}) &:= 
     \int^{\mathbf{x}_{\Delta_n}}_{\mathbf{x}_{\Delta_0}} \nabla F(\psi(\alpha))\cdot \psi^\prime(\alpha) d \alpha.
\end{align*}
Expanding this, we obtain
\begin{align*}
    &\Phi_{\mbox{QUCE}}(\mathbf{x}^{\Delta}) := \\ 
     &\bigg((\mathbf{x}_{\Delta_1} - \mathbf{x}_{\Delta_0}) \times \bigg( \int^{\mathbf{x}_{\Delta_1}}_{\mathbf{x}_{\Delta_0}} \nabla F(\psi(\alpha)) \bigg)\bigg) \\ &+ \ldots + \bigg((\mathbf{x}_{\Delta_n} - \mathbf{x}_{\Delta_{n-1}}) \times \bigg( \int^{\mathbf{x}_{\Delta_n}}_{\mathbf{x}_{\Delta_{n-1}}} \nabla F(\psi(\alpha))\bigg)\bigg).
\end{align*}
Rewriting in terms of partial derivatives we get, 
\begin{align*}
        \Phi_{\mbox{QUCE}}(\mathbf{x}^{\Delta}) := 
     \sum_{j=1}^J \bigg( \int^{\mathbf{x}_{\Delta_1}}_{\mathbf{x}_{\Delta_0}} \frac{\partial F(\psi(\alpha))}{\partial \psi^j(\alpha)}\frac{\partial \psi^j(\alpha)}{\partial \alpha}d \alpha\bigg) \\ + \ldots + \sum_{j=1}^J \bigg( \int^{\mathbf{x}_{\Delta_n}}_{\mathbf{x}_{\Delta_{n-1}}} \frac{\partial F(\psi(\alpha))}{\partial \psi^j(\alpha)}\frac{\partial \psi^j(\alpha)}{\partial \alpha}d \alpha\bigg)
\end{align*}
where $\psi^j$ is along a single feature dimension of a path $\psi$. With this decomposition, we can further extrapolate by defining each step in $\mathbf{x}^{\Delta}$ as a piecewise linear path, such that
\begin{align*}
    &\Phi_{\mbox{QUCE}}(\mathbf{x}^{\Delta}) := \\
     &\sum_{j=1}^J \bigg( (x_{\Delta_{1}}^j - x_{\Delta_{0}}^j) \times \int^1_{\alpha = 0} \frac{\partial F(\mathbf{x}_{\Delta_{0}} + \alpha(\mathbf{x}_{\Delta_{1}} - \mathbf{x}_{\Delta_{0}}) )}{\partial x^j}d \alpha\bigg) \\ &+ \ldots +\\ &\bigg( (x_{\Delta_{n}}^j - x_{\Delta_{n-1}}^j) \times \int^1_{\alpha = 0} \frac{\partial F(\mathbf{x}_{\Delta_{n-1}} + \alpha(\mathbf{x}_{\Delta_{n}} -\mathbf{x}_{\Delta_{n-1}}))}{\partial x^j}d \alpha\bigg)
\end{align*}
and rewriting for a single feature $j$, we simply remove the summation and define the explanation over $x^{\Delta,j}$:
\begin{align*}
    &\Phi_{\mbox{QUCE}}(x^{\Delta,j}) := \\
     &\bigg( (x_{\Delta_{1}}^j - x_{\Delta_{0}}^j) \times \int^1_{\alpha = 0} \frac{\partial F(\mathbf{x}_{\Delta_{0}} + \alpha(\mathbf{x}_{\Delta_{1}} - \mathbf{x}_{\Delta_{0}}) )}{\partial x^j}d \alpha\bigg) \\ &+ \ldots +\\ &\bigg( (x_{\Delta_{n}}^j - x_{\Delta_{n-1}}^j) \times \int^1_{\alpha = 0} \frac{\partial F(\mathbf{x}_{\Delta_{n-1}} + \alpha(\mathbf{x}_{\Delta_{n}} -\mathbf{x}_{\Delta_{n-1}}))}{\partial x^j}d \alpha \bigg).
\end{align*}
We can then rewrite this as a computable Riemann approximation for $K$ steps, as
\begin{align*}
    &\Phi_{\mbox{QUCE}^{\mathcal{R}}}( x^{\Delta,j}) := \\
     &\bigg( (x_{\Delta_{1}}^j - x_{\Delta_{0}}^j) \times \frac{1}{K}\sum_{k=1}^K \frac{\partial F(\mathbf{x}_{\Delta_{0}} + \frac{k}{K}(\mathbf{x}_{\Delta_{1}} - \mathbf{x}_{\Delta_{0}}) )}{\partial x^j}\bigg) \\ &+ \ldots +\\ &\bigg( (x_{\Delta_{n}}^j - x_{\Delta_{n-1}}^j) \times \frac{1}{K}\sum_{k=1}^K  \frac{\partial F(\mathbf{x}_{\Delta_{n-1}} + \frac{k}{K}(\mathbf{x}_{\Delta_{n}} -\mathbf{x}_{\Delta_{n-1}}))}{\partial x^j}\bigg)
\end{align*}
which yields a computable explanation over the $j^{th}$ dimension of an instance $\mathbf{x}$ along a piecewise linear path. This can simply be executed across all features $j$.
\end{proof}
The Riemann approximation is the algorithm for computing explanations that we use in our implementation carried out on each dimension $j$, which returns a vector containing the overall attribution over piecewise linear path integral formulation. It follows that the expected gradients variant can be easily computed by averaging explanations over each counterfactual example in the set $C$. 
\\ 
\\
\textbf{Corollary 1.}
\emph{The expected QUCE variant has a Riemann approximation solution for each feature.}

\begin{proof}
    Given a set of $k$ counterfactual examples in the set $C$, we can simply compute the expected attribution over the $j^{th}$ feature of each generative counterfactual example as
    \begin{align*}
        \Phi_{\mbox{exQUCE}^{\mathcal{R}}}(x^{\Delta,j}) = \frac{1}{k}\sum_{\mathbf{x^c} \in C}\Phi_{\mbox{QUCE}^{\mathcal{R}}}(x^{\Delta,j}).
    \end{align*}
\end{proof}
\begin{table*}[t]
\centering
\caption{Comparison of the deletion scores for counterfactual generative methods that provide feature attribution values. This is experimented over 100 instances on each dataset. Here the lower the value the better. We observe that the proposed QUCE method performs best across a larger fraction of datasets.}
\begin{tabular*}{\textwidth}{c @{\extracolsep{\fill}} ccccccc}
\toprule
\textbf{Deletion} & Lung & Breast & Skin & Lymph & Rectal & COVID & W-BC \\
\midrule 
QUCE & \textbf{0.556} & 0.689 & 0.656 & 0.611 & \textbf{0.669} & \textbf{0.699} & \textbf{0.632}\\
DiCE & 0.561 & 0.688 & \textbf{0.649} & 0.619 & \textbf{0.669} & 0.710 &0.637\\
AGI & 0.559 & \textbf{0.683} & 0.659 & \textbf{0.607} & 0.670 & 0.728 & 0.648\\
\end{tabular*}
\label{table:deletion_metric}
\end{table*}
\section{Proof of Proposition 2}
\textbf{Proposition 2.} \emph{Increasing the $\lambda$-tolerance of uncertainty provides a more flexible search space for possible paths to a generative counterfactual example. }
\begin{proof}
Recall the objective function:
\begin{align*}
\mathcal{G}(\mathbf{x}) &= \underset{\mathbf{x}^c}{\arg \min} \  \lambda_1\mathcal{L}_{pr} + \lambda_2\mathcal{L}_{\delta} + \lambda_3\mathcal{L}_{\epsilon}.
\end{align*}
For simplicity we let $\lambda_1 = 0$, $0 < \lambda_2 \leq 1$ and $0 \leq \lambda_3 \leq 1$,  such that  
\begin{align*}
    &\mathcal{G}(\mathbf{x}) =  \lambda_2\mathcal{L}_{\delta}+ \lambda_3 \mathcal{L}_{\epsilon} \\
    &= \frac{\lambda_2}{2} \vert\vert \mathbf{x}^c - \mathbf{x} \vert\vert^2 + \\ &\lambda_3 (\mathbbm{E}_{q_{\theta^*}}[log \ q_{q_{\theta^*}}(\mathbf{z} \vert \mathbf{x}^c) - log \ p_{\psi^{*}}(\mathbf{z})] - \mathbbm{E}_{q_{\theta^*}} \ log \ p_{\psi^{*}}(\mathbf{x}^c \vert \mathbf{z})). 
\end{align*}
Trivially, as $\lambda_3$ decreases toward zero (we accept more uncertainty), the freedom in the distance function increases, as it not constrained by uncertainty, since  
\begin{align*}
&\lim_{\lambda_3 \rightarrow 0^+} (\frac{\lambda_2}{2} \vert\vert \mathbf{x}^c - \mathbf{x} \vert\vert^2 + \\ &\lambda_3 (\mathbbm{E}_{q_{\theta^*}}[log \ q_{q_{\theta^*}}(\mathbf{z} \vert \mathbf{x}^c) - log \ p_{\psi^{*}}(\mathbf{z})] - \mathbbm{E}_{q_{\theta^*}} \ log \ p_{\psi^{*}}(\mathbf{x}^c \vert \mathbf{z})) ) \\ 
&= \frac{\lambda_2}{2} \vert\vert \mathbf{x}^c - \mathbf{x} \vert\vert^2.
\end{align*}
Here the eigenvalues of the Hessian are given by $\lambda_2$ and since $\lambda_2 > 0$, the Hessian is positive definite and thus the search space is convex, implying the global minima are conditioned on $\mathcal{L}_\delta$ when uncertainty $\mathcal{L}_{\epsilon}$ is relaxed.
\end{proof}
\section{Proof of Proposition 3}
\textbf{Proposition 3.} \emph{Given the function $\Phi_{\mbox{exQUCE}}(\mathbf{x})$, the expected difference in prediction probabilities between generated counterfactuals in the set $C$ with respect to the prediction probability given by $F(\mathbf{x})$, the following equality holds:
}
\begin{align}\label{p1}
&\underset{{\mathbf{x}^c \sim C}}{\mathbbm{E}} \bigg[\Phi_{\mbox{QUCE}}(\mathbf{x}^{\Delta})  \bigg] \\ \label{p2}&= \underset{{\mathbf{x}^c \sim C}}{\mathbbm{E}} \bigg[F(\mathbf{x}^c) - F(\mathbf{x}) \bigg].    
\end{align}
\begin{proof}
Due to the \emph{completeness} axiom the following holds true: 
\begin{align}
F(\mathbf{x}_{\Delta_n}) - F(\mathbf{x}_{\Delta_0}) =  \int^{\mathbf{x}_{\Delta_n}}_{\mathbf{x}_{\Delta_0}} \nabla F(\psi(\alpha))\cdot \psi^\prime(\alpha)   d \alpha
\label{eqn_defn1}
\end{align}
given $\mathbf{x}_{\Delta_n} = \mathbf{x}^c$ and $\mathbf{x}_{\Delta_0} = \mathbf{x}$ respectively, we have 
\begin{align*}
    F(\mathbf{x}^c) - F(\mathbf{x}).
\end{align*}
By relaxing a strict definition of $\mathbf{x}^c$, where we instead use a set of generated counterfactuals $C$, yielding
\begin{align}
 &\int_{\mathbf{x}^c}\bigg(  \int^{\mathbf{x}^c}_{\mathbf{x}} \nabla F(\psi(\alpha))\cdot \psi^\prime(\alpha)  d\alpha\bigg)p_C(\mathbf{x}^c) d\mathbf{x}^c \\ 
 &= \underset{\mathbf{x}^c \sim C, \alpha \sim \mathcal{U}(0,1)}{\mathbbm{E}} \bigg[ \nabla F(\psi(\alpha)) \cdot \psi^\prime(\alpha)   \bigg]
\end{align}
and since $F(\mathbf{x})$ is a constant, we have:
\begin{align}
\underset{{\mathbf{x}^c \sim C}}{\mathbbm{E}} \bigg[F(\mathbf{x}^c)\bigg] - F(\mathbf{x}) =   \underset{{\mathbf{x}^c \sim C}}{\mathbbm{E}} \bigg[F(\mathbf{x}^c) - F(\mathbf{x}) \bigg],   
\end{align}
equations \ref{p1} and \ref{p2} are equivalent.
\end{proof} 


\section{QUCE Evaluated against Further Properties of Explainability}
In the work of \cite{ijcai2022p90} the authors present a desirable set of axiomatic foundations for XAI methods. As a brief informal overview, we consider the following proposed axioms: 
\begin{itemize}
    \item \textbf{Success:} The explainer method should be able to produce explanations for any instance. 
    \item \textbf{Explainability:}  An explanation method should provide informative explanations. An empty explanation here is not recommended.
    \item \textbf{Irreducability:} An explanation should not contain irrelevant information.
    \item \textbf{Representativity:} An explanation should be possible on unseen instances.
    \item \textbf{Relevance:} Information should only be included if it impacts the prediction. 
\end{itemize}
We evaluate QUCE against these properties as interpreted with respect to our model design. We show that it is inherently straightforward to prove that our proposed QUCE method satisfies these desirable axioms. 
\begin{proposition}
The QUCE method can always provide an explanation satisfying the success axiom. 
\end{proposition}
\begin{proof}
    Whilst a counterfactual may not always be \emph{valid}, as a direct implication of the generative learning process QUCE will achieve an explanation.
\end{proof}
\begin{corollary}
    The QUCE method satisifies the Explainability axiom. 
\end{corollary}
\begin{proof}
    Assuming that different instances generated by QUCE do not have the same prediction probability with respect to the target class, it follows immediately from proposition 3 that explainability holds.
\end{proof}
\begin{corollary}\label{corr:irreduce}
    The QUCE method satisfies the Irreducability axiom. 
\end{corollary}
\begin{proof}
    Here, we characterize irrelevance under our own interpretation: since a feature that does not change does not affect the predicted outcome, it should be assigned zero attribution. Then directly from the definition of QUCE it is clear that irreducability holds, as the gradients are multiplied by a zero-value scalar for the same valued features. 
\end{proof}
\begin{proposition}
    The QUCE method satisfies the Representativity axiom.
\end{proposition}
\begin{proof}
    It is easy to see that any instance with the same dimensionality of the instances from a training dataset can utilize the QUCE approach.
\end{proof}
\begin{corollary}
    The QUCE method satisfies the relevance axiom. 
\end{corollary}
\begin{proof}
    Relevance holds as a direct implication of irreducability as seen in corollary \ref{corr:irreduce} and the fact that gradients are traced over the change in the predictions along paths, thereby guaranteeing model-specific relevance.
\end{proof}

\section{Experimental Setup}
For the experiments presented in this paper, the details of hyper-parameters are found in the notebook file at: \url{https://github.com/jamie-duell/QUCE}. For the comparative experiments in this paper we set up QUCE with single-path solutions using the Adam optimizer. The empirical intuition for using Adam for a single-path approach is as follows: as the loss is minimized and points along the path become more reliable there will be an increase in the frequency of points as we approach the solution; thus averaging the gradient along such paths become more reliable. Similarly the approach is deterministic, therefore does not require multiple paths as with alternative optimizers.

\section{Deletion Experiments}
To compare counterfactual feature attribution methods we evaluate the deletion score, a common metric used for evaluating feature attribution methods for identifying important features. The deletion score is used in various studies \cite{Yang_Akhtar_Wen_Mian_2023,pmlr-v202-akhtar23a}; here, a lower value indicates better performance. In Table \ref{table:deletion_metric} we observe that the QUCE method performs better on average than both DiCE and AGI for counterfactual feature attribution performance. 

\section{QUCE Algorithm}
\begin{algorithm}[H]
  \begin{algorithmic}
\State $X$ is a dataset, $F$ is a deep network, $\mbox{VAELoss}(\cdot)$ is a pretrained VAE on $X$, $\mathcal{G}(\cdot)$ is the joint objective function for QUCE, $\mathcal{T}$ is the target class, $\tau$ is a probability threshold towards target class, $K$ is the number of Riemann steps, $n$ is the number of gradient descent updates, $\varphi$ is a learning rate.  
\State $i$=1
\State $\mathbf{x}^\Delta = []$
\Procedure {QUCE}{$\mathbf{x}$}
    \State pass the instance $\mathbf{x}$ into the function $\mathcal{G}$
    \State initialise an instance $\mathbf{x}_c = \mathbf{x}$  
    \While{$i \leq n$}
        \State update $\mathbf{x}_c$ with $\mathbf{x}_c - \varphi \nabla_{\mathbf{x}_c}(\mathcal{G}(\mathbf{x}))$
        \State append updated $\mathbf{x}_c$ to $\mathbf{x}^\Delta$
        \State increment $i$
    \EndWhile
    \If{$F(\mathcal{T}\vert\mathbf{x}_c) \geq \tau$}
    \State pass the output $\mathbf{x}_c$ into $\mbox{VAELoss}$ and return $\epsilon_\mathbf{d}$
    \State take the $K$-step Riemann integral approximation over $\mathbf{x}^\Delta$
    \State take the $K$-step Riemann integral approximation between $\mathbf{x}_c$ and $\mathbf{x}_{c}\pm\epsilon_\mathbf{d}$
    
    \Return {$\mathbf{x}_c$, explanation vector, two uncertainty explanation vectors for $\mathbf{x}_{c} \pm \epsilon$}
    
    \EndIf
\EndProcedure
\caption{Quantified Uncertainty Counterfactual Explanations (QUCE)}\label{alg:QUCE}
\end{algorithmic}
\end{algorithm}

\bibliographystyle{unsrt}  
\bibliography{references}

\end{document}